\documentclass{article}
\usepackage{arxiv}
\usepackage[utf8]{inputenc} 
\usepackage[T1]{fontenc}    
\usepackage{hyperref}       
\usepackage{url}            
\usepackage{booktabs}       
\usepackage{amsfonts}       
\usepackage{nicefrac}       
\usepackage{microtype}      
\usepackage{lipsum}
\usepackage{amsmath}
\usepackage{amssymb}
\usepackage{amsthm}
\usepackage{comment}
\usepackage{bbold}
\usepackage{algorithm}
\usepackage{graphics}
\usepackage{graphicx}
\usepackage{todonotes}
\usepackage{float}
\usepackage[noend]{algpseudocode}%

\newtheorem{definition}{Definition}[section]
\newtheorem{proposition}{Proposition}[section]
\newtheorem{remark}{Remark}[section]

\DeclareMathOperator*{\argmin}{argmin}
\DeclareMathOperator*{\card}{count}

\title{Generalized Quantile Loss for Deep Neural Networks}

\author{
  Dvir Ben Or~~~~~Michael Kolomenkin~~~~~Gil Shabat \thanks{gils@playtika.com}\\ 
  Playtika AI Research, Israel
}

\date{}
\begin{document}
\maketitle

\begin{abstract}
This note presents a simple way to add a count (or quantile) constraint to a regression neural net, such that given $n$ samples in the training set it guarantees that the prediction of $m<n$ samples will be larger than the actual value (the label). Unlike standard quantile regression networks, the presented method can be applied to any loss function and not necessarily to the standard quantile regression loss, which minimizes the mean absolute differences. Since this count constraint has zero gradients almost everywhere, it cannot be optimized using standard gradient descent methods. To overcome this problem, an alternation scheme, which is based on standard neural network optimization procedures, is presented with some theoretical analysis.
\end{abstract}


\section{Introduction}
\label{sec:introduction}
In many applications, it is often required to predict the conditional probability rather than the conditional mean. Among those applications one can find electricity consumption forecasting \cite{he2019electricity}, short term power load forecasting \cite{he2016short} and financial returns \cite{taylor2000quantile} to name some. 

Perhaps the most known tool for tackling those problems is the quantile regression \cite{koenker2001quantile}. While regular least squares minimization estimates the conditional mean, quantile regression estimates the median or any other quantile. Formally, the model parameters are derived by optimizing:

\begin{equation}
\label{eq:quantreg}
    \hat{\beta}=\underset{\beta \in \mathbb{R}^d}{\mbox{argmin}}  \left[(\tau - 1)\sum_{y_{i}<x_i^T\beta}(y_{i}-x_i^T\beta)+\tau\sum_{y_{i}\geq x_i^T\beta}(y_{i}-x_i^T\beta) \right]
\end{equation}
where $\{x_i\}_{i=1}^n \in \mathbb{R}^d$ are the data points, $\{y_i\}_{i=1}^n \in \mathbb{R}$ are the labels, $0 < \tau <1$ is the desired quantile and $\beta$ are the model parameters to be determined.
Recently, it was suggested to optimize the quantile loss using more complex models and to utilize the quantile loss to estimate uncertainty of neural networks \cite{rodrigues2020beyond}. In this case, $x_i^T\beta$ is replaced by some neural network $N\left(x_i;\mathbf{\Theta}\right)$ with parameters $\mathbf{\Theta}$ and Eq. \ref{eq:quantreg} is used as the loss function, usually reformulated a little bit different.

One of the limitations of general quantile-based predictions is that they are ill-posed, since there is an infinite number of possibilities for fitting a curve to pass between the data points in a way that a certain number of points is above the curve and the rest of the points are below the curve. As for the quantile loss of quantile regression, it minimizes the mean of absolute differences. For example, for predicting the conditional probability of the $50\%$ quantile, subtituting $\tau=0.5$ leads to 
\begin{equation}
    \hat{\mathbf{\Theta}}=\underset{\beta \in \mathbb{R}^d}{\mbox{argmin}}\sum_{i=1}^n \vert y_i-N\left(x_i;\mathbf{\Theta}\right) \vert.
    \label{eq:minmad}
\end{equation}
Eq. \ref{eq:minmad} is a standard $l_1$ loss function for minimizing the mean absolute differences (MAD). Therefore, it passes a manifold that separates between the dataset, so half of the dataset is above that manifold and half is it, in such a way that the sum of $l_1$ distances between the manifold and the data points will be minimal. Changing the formulation such that the goal will be to minimize a different loss function (other than $l_1$) subject to a quantile constraint is not straightforward. In this note, we present a computational flow that generalizes the quantile regression for neural networks, such that the loss function can be almost any loss function and yet to be able to add constraint on the quantile of the model.
 
\section{Problem Formulation}
Given a regression loss function of a neural network, and a number (or percentile) of samples, the suggested method optimizes the neural net to minimize the loss function such that only a specified fraction of samples will be above the predicted regression value of the neural network. Formally, given $n$ samples, $x_i \in \mathbb{R}^d$, $i=1,\ldots,n$ with their corresponding values, $y_i \in \mathbb{R}$ and a number $m \le n$, the algorithm find weights $\Theta$ in order to optimize:

%
\begin{align}
\label{eq:problem}
    &\mbox{minimize}~~~\mathcal{L}\left(N(X,\mathbf{\Theta});Y\right) \\
    &\mbox{such that: ~~count}_i\left(\hat{y}_i \ge y_i\right) = m \nonumber
\end{align}

where $N(X;\mathbf{\Theta})$ is the output of the neural network with parameters $\mathbf{\Theta}$ given input data $X$. The predicted value is denoted by $\hat{y}$, i.e. $\hat{y} \triangleq N(x_i;\mathbf{\Theta})$.
The \emph{count} function returns the amount of times the condition inside holds and can be defined via the indicator function, i.e.
\begin{equation*}
    \card_i\left(a_i>b_i\right) = \sum_i \mathbb{1}_{a_i > b_i}
\end{equation*}
One of the challenges Eq. \ref{eq:problem} holds is dealing with the count function, which has no gradients and cannot be optimized directly using standard gradient methods.
An example that can be given as a special case of Eq. \ref{eq:problem} is to minimize the mean-squared-error (MSE) over $n$ samples, such that the error of $10\%$ of the predicted values will be above the real values:

\begin{align*}
    &\mbox{minimize}~~~\sum_{i=1}^n \left(N\left(x_i;\mathbf{\Theta}\right)-y_i\right)^2 \\
    &\mbox{such that: ~~} \mbox{count}_i \left(N(x_i;\mathbf{\Theta}) \ge y_i\right) = 0.1n
\end{align*}

\begin{remark}
The accuracy for satisfying the count-constraint is implemented up to some tolerance $\delta$, so that the constraint in Eq. \ref{eq:problem} is now $\vert \mbox{count}_i\left(\hat{y}_i \ge y_i\right)-m \vert \le \delta$
\end{remark}
\begin{align*}
    &\mbox{minimize}~~~f\left(N(X;\mathbf{\Theta}),Y\right) \\
    &\mbox{such that: ~~} \vert \mbox{count}_i\left(\hat{y}_i \ge y_i\right) - m \vert  \le \delta
\end{align*}

\subsection{Description of the Algorithm}
The algorithm consists of optimizing Eq. \ref{eq:problem} using alternations. The alternations can be viewed as a non-linear and non-orthogonal projection operators. The first type of alternation moves from the current weights $\Theta$ to the closest set of weights $\hat{\Theta}$ which is a local minima of the loss function. This operator is denoted by $\mathcal{P_M}$.

\begin{definition}
\label{def:ProjMin}
Given a training dataset $\mathbf{X}$, 
a neural net $N(\boldsymbol{\Theta}, \mathbf{X})$ with weights $\boldsymbol{\Theta}$ and a loss function $\mathcal{L}(\Theta,X;Y)$ with a set of local minima $\mathcal{M}$, then  $\mathcal{P_M}\mathbf{\Theta}$ returns weights that are the nearest local minimum:
\begin{equation}
    \mathcal{P_M}\mathbf{\Theta} =  \argmin_{\mathcal{L}(N(\boldsymbol{\hat{\mathbf{\Theta}},X}); \mathbf{Y}) \in \mathcal{M}} \Vert \mathbf{\Theta} - \hat{\mathbf{\Theta}} \Vert_2
\end{equation}
\end{definition}

\begin{definition}
\label{def:SetComp}
Let $\mathcal{C}$ be the set of all possible weights of a loss function $\mathcal{L}(N(\boldsymbol{\Theta},X);Y)$, such that the predictions of the neural net,  $\{\hat{y_i}\}_{i=1}^n$, will be above the real values $\{y_i\}_{i=1}^n$ for $m$ points in the training dataset, up to tolerance $\delta$. Formally:
\begin{equation}
\mathcal{C} = \Bigl\{\boldsymbol{\Theta} \vert ~~~ \vert  \mbox{count}_i\left(N(\boldsymbol{\Theta}, x_i) \geq y_i \right)  - m \vert \le \delta\Bigr\}  
\end{equation}
\end{definition}

\begin{definition}
\label{def:ProjComp}
Given a training dataset $\mathbf{X}$, 
a neural net $N(\boldsymbol{\Theta}, \mathbf{X})$ with weights $\boldsymbol{\Theta}$ and a set of valid count-constraint points $\mathcal{C}$ (Def. \ref{def:SetComp}), then $\mathcal{P_C} (\Theta)$ returns the closest weights in $\mathcal{C}$:
\end{definition}
\begin{equation}
    \mathcal{P_C}(\boldsymbol{\Theta}) =  \argmin_{\boldsymbol{\hat{\Theta}} \in \mathcal{C}} \Vert \mathbf{\Theta} - \hat{\mathbf{\Theta}} \Vert_2
\end{equation}

\begin{remark}
Note that $\mathcal{P_C}$ does not depend on the loss function $\mathcal{L}$
\end{remark}

Both $\mathcal{P_C}$ and $\mathcal{P_M}$ are implemented using stochastic gradient optimizers (specifically, in this paper Adam optimizer was used). The implementation of $\mathcal{P_M}$ is a standard neural network optimization. $\mathcal{P_C}$ is implemented by drifting iteratively from the current point $\Theta$ to a valid point $\mathcal{P_C}\Theta \in \mathcal{C}$, such that if $\card_i\left(\hat{y_i}>y_i\right)$ is too large, $\Theta$ is moved against the direction of the gradient of $N(\boldsymbol{\hat{\Theta}}, \mathbf{X})$ to reduce the count value or with the direction of the gradient to increase the count function, if the count value is too small. 

In practice, implementation of the above operators such that they return the nearest minimum is impossible in general, since it depends on the data, the architecture of the network and the loss function, which is typically a high-dimensional non-convex manifold. However, since the operators are implemented using stochastic gradient descent (or other optimizers), it is likely to assume that the weights returned by the operators are close (probably among the closest) to the point the operator started from.

Optimizing the loss function subject to a count constraint, can be done by the following alternating scheme, which is approximately implemented by Algorithm \ref{alg:projcount}:
\begin{equation}
    \label{eq:alt1}
    \Theta_i^M \leftarrow \mathcal{P_M}\Theta_i^C
\end{equation}

\begin{equation}
    \label{eq:alt2}
    \Theta_{i+1}^C \leftarrow \mathcal{P_C}\Theta_i^M
\end{equation}
where $\Theta_0^C$ is an arbitrary starting point (random initialization of the weights).

\begin{algorithm}
    \begin{algorithmic}[1]
    \Require
        $N(\boldsymbol{\Theta};\mathbf{X})$ - neural network,
        $\boldsymbol{\Theta}$ - current network weights,
        $\mathbf{Y} = \{y_1, \ldots, y_n\} \in \mathbb{R}$ - labels,
        $\mathbf{X} = \{ x_1, \ldots, x_n \} \in \mathbb{R}^d$ - Input data,
        $m$ - Number of samples to satisfy count constraint,
        $\delta$ - Tolerance for the count error,
        $\mu$ - Learning rate,
        $\mbox{MaxIter}$ - Maximal number of iterations.
    \Ensure 
         $\hat{\boldsymbol{\Theta}}$ - optimized network weights
    \State $i \leftarrow 0$
    \State $\hat{\boldsymbol{Y}} \leftarrow N(\boldsymbol{\Theta,X})$ \# Compute predicted labels
    \State $\hat{m} \leftarrow \card_i{\left(\hat{y_i} > y_i \right)}$
      \While{$\left(\vert \hat{m}-m \vert > \delta\right)$ AND $\left(i < \mbox{MaxIter}\right)$}
    \If{$\hat{m}>m$}  ~~~\# Not enough samples passed, increase value
        \State $L \leftarrow \frac{1}{n}\sum_{i=1}^n \hat{y_i}$
    \Else
        \State $L \leftarrow -\frac{1}{n}\sum_{i=1}^n \hat{y_i}$
    \EndIf
  \State $\boldsymbol{\Theta} \leftarrow \boldsymbol{\Theta} - \mu \nabla_\Theta L$ ~~~~~\# Optimize over batch
  \State $i \leftarrow i+1$
    \EndWhile
     \State \Return $\boldsymbol{\Theta}$
    \end{algorithmic}
    \caption{Satisfy count constraint, implementation of $\mathcal{P_C}$ operator}
    \label{alg:projcount}
\end{algorithm}

\begin{proposition}
\label{prop:conv}
Let $\Theta_i^C$ and $\Theta_i^M$ ($i \ge 1$) be a set of points (weights) obtained by a consecutive application of the alternation scheme (Eqs. \ref{eq:alt1} and \ref{eq:alt2}) then the series $\Vert \Theta_i^C - \Theta_i^M \Vert$ converges.
\end{proposition}

\begin{proof}
Since $i \ge 1$, then according to Eq. $\ref{eq:alt2}$, $\Theta_i^C \in \mathcal{C}$ (Def. \ref{def:SetComp}). By the definition of $\mathcal{P_M}$, $\Theta_i^M$ is the closest local minima to $\Theta_i^C$ and by the definition of $\mathcal{P_C}$, $\Theta_{i+1}^C$ is the closest valid count-constraint point to $\Theta_i^M$. Since $\Theta_i^C \in \mathcal{C}$ and $\Theta_{i+1}^C \in \mathcal{C}$ is the closest point to $\Theta_i^M$
\begin{equation}
\label{eq:propeq1}
    \Vert \Theta_i^M - \Theta_{i+1}^C \Vert \le \Vert \Theta_i^M - \Theta_i^C \Vert.
\end{equation}
By the definition of $\mathcal{P_M}$, $\Theta_{i+1}^M \in \mathcal{M}$ is the closest local minima to $\Theta_{i+1}^C$. Since $\Theta_i^M \in \mathcal{M}$
\begin{equation}
\label{eq:propeq2}
    \Vert \Theta_{i+1}^C - \Theta_{i+1}^M \Vert \le \Vert \Theta_i^M - \Theta_{i+1}^C \Vert
\end{equation}
Combining Eqs. \ref{eq:propeq1} and \ref{eq:propeq2} gives
\begin{equation*}
    \Vert \Theta_{i+1}^M - \Theta_{i+1}^C \Vert \le \Vert \Theta_i^M - \Theta_i^C \Vert.
\end{equation*}
Since $\Vert \Theta_i^M - \Theta_i^C \Vert$ is monotonically decreasing and bounded it converges, which completes the proof.
\end{proof}
Proposition \ref{prop:conv} states that the distance between a valid count-constraint point and a local minima point is monotonically decreasing and eventually converges. An interesting observation from the proposition is that it tells us where to look for the next minima/valid count-constraint point, which enables to decrease the step size of the SGD proportionally to the distance between the two points. The proposition is illustrated in Fig. \ref{fig:moon_conv}.

Additionally, the following observations infer directly from Proposition \ref{prop:conv}:
\begin{itemize}
    \item Since the distance between a valid count-constraint point and a local minimum converges, then eventually it means (excluding pathological cases of points having \emph{exactly} the same distance) that the algorithm iterates between one local minimum and one valid count-constraint point. Therefore, it converges to a specific local minimum/count-constraint point.
    \item The difference in model's performance between those two points, depends on the distance and the Lipschitz constant of the neural network \cite{fazlyab2019efficient}. So if the distance is small (and hopefully the Lipschitz constant), then stopping in count-constraint  point or in a local minimum should not make a big difference.
\end{itemize}

\begin{figure}
  \centering
    \includegraphics[width=0.7\textwidth]{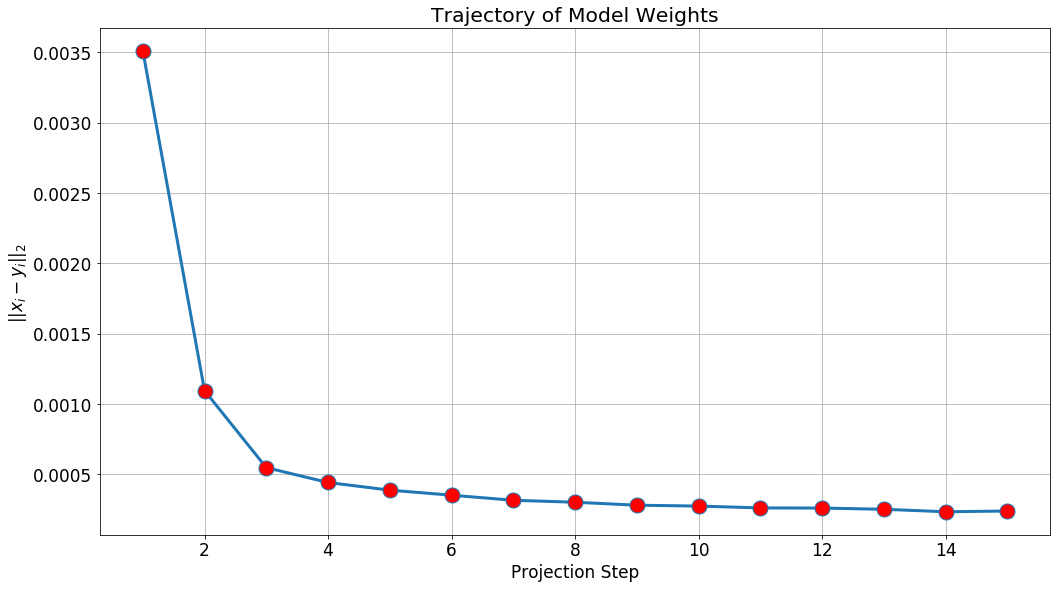}
    \caption{The distance between the weights of a valid count-constraint point to the local minima followed by the application of $\mathcal{P_C}$ to it}
    \label{fig:moon_conv}
\end{figure}

\section{Results}
\subsection{Motorcycle Dataset}
In this subsection, the algorithm was applied to the motorcycle dataset \cite{silverman1985some} to minimize the MSE over several
percentiles: $25\%, 40\%$ and $60\%$. The neural network is a simple two layers fully connected layers, the first hidden layer has 50 neurons with $\tanh$ activation function, following by a layer with 10 neurons followed by a ReLU activation function.

\begin{figure}
  \centering
    \includegraphics[width=0.85\textwidth]{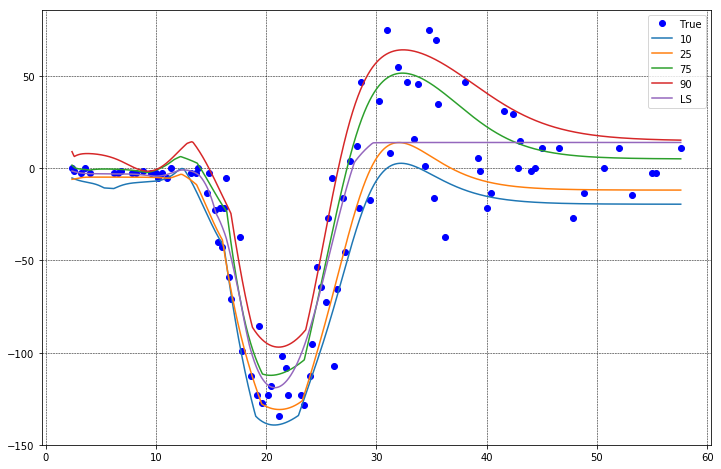}
    \caption{A variety of curves that minimizes the MSE under different quantile constraints, and regular least squares minimization.}
    \label{fig:motorcycle}
\end{figure}

\begin{table}[H]
\centering
\begin{tabular}{|l|l|}
\hline
\multicolumn{1}{|c|}{\textbf{\% Above Data}} & \multicolumn{1}{c|}{\textbf{RMSE}} \\ \hline
10\%                                         & 29.8                               \\ \hline
25\%                                         & 25.1                               \\ \hline
75\%                                         & 23.22                              \\ \hline
90\%                                         & 31.6                               \\ \hline
MSE Minimization                             & 22.9                               \\ \hline 
\end{tabular}\\
\caption{RMSE Error of the curves from Figure \ref{fig:motorcycle} with respect the real data}
\label{table:results}
\end{table}

Table \ref{table:results} shows the error between the model and the real data, i.e. $\mbox{RMSE} = \sqrt{\frac{1}{n}\sum_{i=1}^n \left(\hat{y_i}-y_i\right)^2}$

\section{Conclusion}
This note presented an algorithm that trains a neural network to minimize a general loss function under quantile constraint, which is difficult to implement straightforward since it has no gradients. The note presented the formulation of the problem, an algorithmic description and some theoretical analysis of why the method converges. Finally, we presented results on a small toy dataset, demonstrating the performance of the algorithm. 
\bibliographystyle{unsrt}  
\bibliography{references}  


\end{document}